\documentclass[11pt]{article}

\usepackage{fullpage}
\usepackage{amsmath,amsbsy,amsfonts,amssymb,amsthm,bm}
\usepackage{graphicx}
\usepackage[margin=1.2in]{geometry} 
\usepackage[colorlinks=True, citecolor=blue]{hyperref}
\usepackage{todonotes}
\usepackage{authblk}
\usepackage{color}

\newtheorem{theorem}{Theorem}[section]

\newtheorem{lemma}[theorem]{Lemma}

\theoremstyle{definition}
\newtheorem{definition}[theorem]{Definition}

\newtheorem{remark}[theorem]{Remark}

\numberwithin{equation}{section}

\DeclareMathOperator{\argmin}{argmin}

\DeclareMathOperator{\rad}{Rad}
\DeclareMathOperator{\lip}{Lip}

\newcommand{\RR}{\mathbb{R}}
\newcommand{\EE}{\mathbb{E}}
\newcommand{\PP}{\mathbb{P}}

\renewcommand{\SS}{\mathbb{S}}

\newcommand{\cH}{\mathcal{H}}
\newcommand{\cP}{\mathcal{P}}
\newcommand{\cF}{\mathcal{F}}

\newcommand{\cB}{\mathcal{B}}

\newcommand{\cD}{\mathcal{D}}
\newcommand{\cC}{\mathcal{C}}

\newcommand{\ba}{\bm{a}}
\newcommand{\bb}{\bm{b}}

\newcommand{\be}{\bm{e}}

\newcommand{\bn}{\bm{n}}

\newcommand{\br}{\bm{r}}
\newcommand{\bx}{\bm{x}}
\newcommand{\by}{\mathbf{y}}
\newcommand{\bz}{\bm{z}}

\newcommand{\bv}{\bm{v}}
\newcommand{\bw}{\bm{w}}

\newcommand{\bbeta}{\bm{\beta}}
\newcommand{\balpha}{\bm{\alpha}}

\title{The Generalization Error of the Minimum-norm Solutions for Over-parameterized Neural Networks}

\author[1,3]{Weinan E \thanks{\texttt{weinan@math.princeton.edu}}}
\author[2]{Chao Ma \thanks{\texttt{chaoma@stanford.edu}}}
\author[3]{Lei Wu \thanks{\texttt{leiwu@princeton.edu}}}

\affil[1]{Department of Mathematics, Princeton University}
\affil[2]{Department of Mathematics, Stanford University}
\affil[3]{Program in Applied and Computational Mathematics, Princeton University}

\begin{document}

\maketitle

\begin{abstract}
We study the generalization properties of minimum-norm solutions for three over-parametrized machine learning models including the random feature model, the two-layer neural network model  and the residual network model. 
We proved that for all three models, the generalization error for the minimum-norm solution is comparable to the 
Monte Carlo rate,
up to some logarithmic terms, as long as the models are sufficiently over-parametrized.
\end{abstract}


\section{Introduction}

Consider a supervised learning problem with training data  $\{(\bx_i, y_i)\}_{i=1}^n$ where $y_i = f^*(\bx_i), i=1, \cdots, n$. 
 Let $\cF$ be the hypothesis space. 
Our objective is to find the best approximation to the target function $f^*$ in the hypothesis space by using only information from the training dataset.
We will assume that $\cF$ is large enough to guarantee that  interpolation is possible, i.e., there exists trial functions $f$
in $\cF$ such that $f(\bx_i) = y_i$ holds for all $i=1, \cdots, n$.  Trial functions that satisfy this condition are called ``interpolated solutions''.
We are interested in the generalization properties of the following minimum-norm interpolated solution:
\begin{equation}\label{eqn: minimum-norm}
\begin{aligned}
\hat{f} \,\in&\, \argmin_{ f \in \cF}\, \|f\|_{\cF} \\
        &\text{s.t.}\,\,\, f(\bx_i) = y_i,\quad i=1, \dots, n.
\end{aligned}
\end{equation}
Here $\|\cdot\|_{\cF}$ is a norm imposed for the model, which {is usually} different for different models. This type of estimators are 
 relevant for understanding various explicitly or implicitly regularized models and optimization methods.  For example, it is well known that for (generalized) linear models, gradient descent converges to minimum $l_2$-norm solutions, if we initialize all the coefficients from zero.  
 In any case, minimum-norm solutions play an important role in the analysis of modern machine learning models
 and they are the focus on this paper.

Assume that the training data $(\bx_i,y_i) \in \RR^{d}\times \RR, i=1, \dots, n$ are generated from the model 
\[
    y_i = f^*(\bx_i), \quad i=1,\dots, n,
\]
where $\bx_i\sim P_{\bx}$ and the random draws are independent.  
$f^*$ is the target function that we want to estimate from the $n$ training samples. In this paper, we will always assume that $\|\bx_i\|_{\infty}\leq 1, |y_i|\leq 1$. Let $X=(\bx_1,\dots,\bx_n)\in\RR^{d\times n}$ and $\by=(y_1,\dots,y_n)^T\in\RR^n$.

Let $f(\cdot;\theta)$ denote the  parametric model, which could be a random feature model, two-layer or residual neural network  model in our subsequent analysis. 
We want to find $\theta$ that minimizes the generalization error (also called the population risk)
\[
    R(\theta) := \EE_{\bx,y}[\ell(f(\bx;\theta), y)].
\]
Here $\ell(y,y') = \frac{1}{2}(y-y')^2$ is the loss function. 
But in practice, we can only deal with the risk defined on the training samples, the empirical risk:
\[
    \hat{R}_n(\theta):=\frac{1}{n}\sum_{i=1}^n \ell(f(\bx_i),y_i).
\]
A key question in machine learning is the size of the population risk (or the generalization error) for minimizers of the empirical risk.
In the case of interest here, the minimizers of the empirical risk are far from being unique,
and we will focus on the one with a minimum norm.

We will consider three classes of models.  Arguably they are the most representative models in the analysis of modern
machine learning algorithms.

\paragraph*{The random feature model.}
Let $\{\varphi(\cdot; \bw), \bw \in \Omega \}$  be a set of random features over some probability space $\Omega$ endowed with a probability measure
$\mu$.
The random feature model is given by 
\begin{equation}\label{eqn: rdf-model}
f_m(\bx;\ba) := {\frac{1}{m}}\sum_{j=1}^m a_j \varphi(\bx; \bw_j),
\end{equation}
where $\ba=(a_1,\dots,a_m)^T \in\RR^m$ are the parameters to be learned from the data, and $\{\bw_j\}_{j=1}^m$ are i.i.d. random variables drawn from  $\mu$. For this model, there is a naturally related reproducing kernel Hlibert space (RKHS) \cite{aronszajn1950theory} $\cH_k$ with the kernel defined by 
\begin{equation}
k(\bx,\bx') := \EE_{\bw\sim \mu} [\varphi(\bx;\bw)\varphi(\bx';\bw)].
\end{equation}
For simplicity, we assume that $|\varphi(\bx;\bw)|\leq 1$. 

Define two kernel matrices $K=(K_{i,j}), K^m=(K^m_{i,j})\in \RR^{n\times n}$ with 
\[
    K_{i,j} = k(\bx_i,\bx_j), \quad K^m_{i,j} = \frac{1}{m}\sum_{s=1}^m \varphi(\bx_i;\bw_s)\varphi(\bx_j;\bw_s),
\]
The latter is an approximation of the former. 

\paragraph*{The two-layer neural network model.}
A two-layer neural network is given by 
\begin{equation}\label{eqn: 2nn}
    f_{m}(\bx;\theta) = {\frac{1}{m}}\sum_{j=1}^{m}a_j\sigma(\bb_j\cdot \bx + c_j).
\end{equation}
Here  $\sigma(t) = \max(0,t)$ is the rectified linear unit (ReLU) activation function.
Let $\theta = \{(a_j,\bb_j, c_j)\}_{j=1}^m$ be all the parameters to be learned from the data. 

If we define $\varphi(\bx;\bb, c) := \sigma(\bb\cdot \bx + c)$, the two-layer neural network is almost the same as the random feature model \eqref{eqn: rdf-model}. The only difference is that $\{\bw_j\}_{j=1}^m$ in the random feature model is fixed during the training process, while the parameters $\{(\bb_j,c_j)\}_{j=1}^m$ in the two-layer neural network model are learned from the data.

Consider the case where $(\bb,c)\sim \pi_0$, where $\pi_0$ is a fixed probability distribution.  We define  $k_{\pi_0}(\bx,\bx')=\EE_{\pi_0}[\sigma(\bb\cdot\bx+c)\sigma(\bb\cdot\bx'+c)]$ and the corresponding kernel matrix $K_{\pi_0} = (k_{\pi_0}(\bx_i,\bx_j))\in\RR^{n\times n}$. Let $\lambda_n = \lambda_n(K_{\pi_0})$, the smallest eigenvalue of $K_{\pi_0}$, which will be used to bound the network width in our later analysis.

\paragraph*{\textbf{Residual neural networks.}}
Consider the following type of residual  neural networks
\begin{equation}\label{eqn: resnet}
\begin{aligned}
\bz_0(\bx) &= V\tilde{\bx} \\
\bz_{l+1}(\bx) &= \bz_{l}(\bx) + {\frac{1}{L}}U_l \sigma(W_l\bz_{l}(\bx)) , \quad l=0,\dots, L-1\\
f_L(\bx;\theta) &= \balpha^T\bz_L(\bx)
\end{aligned}
\end{equation}
where $\tilde{\bx}=(\bx^T,1)^T\in\RR^{d+1}, W_l \in \RR^{m\times D}, U_l\in \RR^{D\times m}, \balpha\in \RR^{D}$ and 
\[
    V = 
    \begin{pmatrix}
        I_{d+1} \\
        0      
    \end{pmatrix} \in\RR^{D\times (d+1)}.
\]
We use $\theta = \{W_1,U_1,\dots, W_L, U_L, \balpha\}$ to denote all the parameters to be learned from the training data. 
To explicitly show the dependence on the hyper-parameters, we call $f_L(\cdot;\theta)$  a $(L,D,m)$ residual network.

There is a large volume of literature on the theoretical analysis of these models. The most important issue is to estimate the generalization error for
the situation when $d$ is large. In this case, one benchmark for us is the Monte Carlo algorithm. Our hope would be to establish estimates
that are comparable to those for the Monte Carlo algorithms.  We call this Monte Carlo-like error rates.
In this regard,  Monte Carlo-like estimates of the generalization error were established in \cite{e2018priori} for two-layer neural network
models and in \cite{ma2019priori} for residual network models, when suitable regularization terms are added explicitly to the model.
These regularization terms help to guarantee the boundedness of certain norms which in turn help to control the generalization gap.
{It should be noted that as is the case for integration problems, small improvements can be made on these rates \cite{dick2013high}, typically from $O(n^{-1/2})$ to $O(n^{-1/2-1/d})$.  However, these improvements become negligible when $d\gg 1$. In general one should not expect asymptotically better than Monte Carlo-like rates in high dimensions.}

For interpolated solutions, recent literature on their mathematical analyses  includes work on the nearest neighbor scheme \cite{belkin2018overfitting}, linear regression \cite{bartlett2019benign,hastie2019surprises}, kernel (ridgeless) regression \cite{belkin2018understand,liang2018just, rakhlin2018consistency,liang2019risk} and random feature model \cite{hastie2019surprises}.  

We will study minimum-norm interpolated solution for the three models described above.
{We} prove that the minimum-norm estimators can achieve the 
Monte Carlo rate up to logarithmic terms, as long as the target functions are in the right function spaces and the models are sufficiently over-parametrized.
More precisely,  we prove the following results.
\begin{itemize}
\item For the random feature model, the corresponding function space is the  reproducing kernel Hilbert space (RKHS) associated with
the corresponding kernel. Optimal rate for the generalization error is proved for the  $l_2$ minimum-norm interpolated solution when the 
model is sufficiently over-parametrized.
\item The  same result is proved for two-layer neural network models. The corresponding function space for the two-layer neural network
model is the Barron space define in \cite{e2019barron,e2018priori}. {Naturally the norm used in \eqref{eqn: minimum-norm} is the Barron norm
(note that the Barron norm is different from the Fourier transform-based norm used in Barron's original paper \cite{barron1993universal}).}
\item The same result is also proved for deep residual network models for which the corresponding function space is  the flow-induced function spaces defined in \cite{e2019barron}, and the
norm used in \eqref{eqn: minimum-norm} is the flow-induced norm.

\end{itemize}
We remark that over-parametrization is a key for these results. {This can be seen from the work \cite{belkin2019reconciling}, which} experimentally showed that  minimum-norm interpolated solutions may generalize very badly if the model is not sufficiently over-parametrized.  In contrast,  the corresponding explicitly regularized models are always guaranteed to achieve the optimal rate \cite{e2018priori,caponnetto2007optimal,e2019barron}. 


To control the gap between population risk and empirical risk, the following notion of complexity of function spaces will be used.

\begin{definition}[Rademacher complexity]
Recall that  $\cF$ and $\{\bx_i\}_{i=1}^n$ denote the hypothesis space and the training data set respectively. 
The Rademacher complexity \cite{shalev2014understanding} of $\cF$ with respect to the data  is defined by 
\[
\rad_n(\cF) := \frac{1}{n}\EE_{\xi_1,\dots,\xi_n} [\sup_{f\in \cF} \sum_{i=1}^n \xi_if(\bx_i)],
\]
where $\{\xi_i\}_{i=1}^n$ are i.i.d random variables with $\PP(\xi=1)=\PP(\xi=-1)=\frac{1}{2}$.
\end{definition}
We will use the following theorem to bound the generalization error.
\begin{theorem}[Theorem 26.5 of \cite{shalev2014understanding}]\label{thm: rad-gen-err}
Assume that the loss function $\ell(\cdot,y')$ is $Q$-Lipschitz continuous and bounded from above by $C$. For any $\delta\in (0,1)$,  with probability $1-\delta$ over the random sampling of the training data, the following generalization bound hold for any $f\in \cF$,
\[
R(f) \leq \hat{R}_n(f) + 2 Q \rad_n(\cF) + 4C \sqrt{\frac{2\ln(2/\delta)}{n}}.
\]
\end{theorem}

\paragraph*{Notation.} We use $\|\bv\|_q$ to denote the standard $\ell_q$ norm of a vector $\bv$, and $\|\cdot\|$ to denote the $l_2$ norm.  For a matrix $A$,  we use $\lambda_j(A)$ to denote the $j$-th largest eigenvalue of $A$ and we also define the norm $\|A\|_{1,1}=\sum_{i,j}|a_{i,j}|$. The spectral and Frobenius norms of a matrix are denoted by $\|\cdot\|$ and $\|\cdot\|_F$, respectively.   We use $X\lesssim Y$ to mean that there exists a universal constant $C>0$ such that $X\leq C Y$. For any positive integer $d$, we let $\SS^{d}:=\{\bw | \bw\in\RR^{d+1}, \|\bw\|_1 = 1\}$ and use $\pi_0$ to denote the uniform distribution over $\SS^d$. For two matrices $A=(a_{i,j}), B=(b_{i,j})$ in $\RR^{n\times m}$,  if $a_{i,j}\leq b_{i,j}$ for any $i\in [n], j \in [m]$, then we write $ A \preceq B$. For any positive integer $q$, we denote by $[q]:=\{1,\dots, q\}$, $\bm{1}_q = (1,\dots,1)\in\RR^q$. For a scalar function $g:\RR\to\RR$ and matrix $A=(a_{i,j})$, we let $g(A)=(g(a_{i,j}))$.

\section{The random feature model}
Consider the minimum $l_2$ norm solution defined by
\begin{equation}\label{eqn: rf-min-norm}
\begin{aligned}
\hat{\ba}_n := \argmin_{\hat{R}_n(\ba)=0} \|\ba\|^2.
\end{aligned}
\end{equation}
About this estimator, we have the following theorem.
\begin{theorem}\label{thm: random-feature}
Assume that  $f^*\in \cH_k$. For any $\delta\in (0,1)$, assume that $m\geq \frac{8n^2\ln(2n^2/\delta)}{\lambda_n^2(K)}$. Then with probability at least $1-\delta$ over the random sampling of the data and the features, we have 
\[
    R(\hat{\ba}_n)\lesssim \frac{\|f^*\|_{\cH_k}^2+1}{\sqrt{n}}\left(1 + \sqrt{\ln(2/\delta)}\right).
\]
\end{theorem}

To prove the above theorem, we need the following lemma, which says that the two kernel matrices are close when the random feature model is sufficiently over-parametrized.
\begin{lemma}\label{eqn: kernel-approx}
For any $\delta \in (0,1)$, with probability $1-\delta$ over the random sampling of $\{\bw_j\}_{j=1}^m$, we have 
\[
    \|K-K^m\|\leq  \sqrt{\frac{n^2\ln(2n^2/\delta)}{2m}}.
\]
In particular, if $m\geq \frac{2n^2\ln(2n^2/\delta)}{\lambda_n^2(K)}$, we have 
\[
    \lambda_n(K^m) \geq \frac{\lambda_n(K)}{2}.
\]
\end{lemma}
\begin{proof}
According to the Hoeffding's inequality, we have that for any $\delta' \in (0,1)$, with probability $1-\delta'$ the following holds for any
specific $i,j \in [n]$,
\[
    |k(\bx_i,\bx_j)-\frac{1}{m}\sum_{j=1}^m \varphi(\bx_i;\bw_j)\varphi(\bx_j;\bw_j)|\leq \sqrt{\frac{\ln(2/\delta')}{2m}}.
\]
Therefore, with probability $1-n^2\delta'$, the above inequality holds for all $i,j\in [n]$. Let $\delta=n^2\delta'$, the above can be written as 
\[
|k(\bx_i,\bx_j)-\frac{1}{m}\sum_{k=1}^m \varphi(\bx_i;\bw_k)\varphi(\bx_j;\bw_k)|\leq \sqrt{\frac{\ln(2n^2/\delta)}{2m}}.
\]
Thus we have 
\[
    \|K-K^m\| \leq \|K-K^m\|_F \leq \sqrt{\frac{n^2\ln(2n^2/\delta)}{2m}}.
\]
Using Weyl's inequality, we have 
\[
    \lambda_n(K^m)\geq \lambda_n(K) - \|K-K^m\| \geq \lambda_n(K) - \sqrt{\frac{n^2\ln(2n^2/\delta)}{2m}}.
\]
When $m\geq \frac{2n^2\ln(2n^2/\delta)}{\lambda_n^2(K)}$, we have $\lambda_n(K^m)\geq \frac{\lambda_n(K)}{2}$.
\end{proof}

We first have the following estimate for  kernel (ridgeless) regression. 
\begin{lemma}\label{lemma: bound of krr}
$\by^T K^{-1}\by \leq \|f^*\|^2_{\cH_k}$.
\end{lemma}
\begin{proof}
Consider the following optimization problem 
\begin{equation}\label{eqn: pro}
    \hat{h}_n = \argmin_{\hat{R}_n(h)=0} \|h\|_{\cH_k}^2.
\end{equation}
According to the Representer theorem (see Theorem 16.1 of \cite{shalev2014understanding}), we can write $\hat{h}_n$ as follows
\[
    \hat{h}_n = \sum_{i=1}^m \beta_i k(\bx_i,\cdot).
\]
Plugging it into $\hat{R}_n(\hat{h}_n)=0$ gives us that 
$
    \by = K \bbeta,
$
which leads to $\bbeta = k^{-1}\by$. According the Moore-Aronszajn theorem \cite{aronszajn1950theory}, we have 
\[
    \|\hat{h}_n\|^2_{\cH_k} = \bbeta^T K \bbeta = \by^T K^{-1} \by.
\]
By  definition  $\hat{h}_n$ is the minimum RKHS norm solutions and  $\hat{R}_n(f^*)=0$,  it follows that  
$
\|\hat{h}_n\|_{\cH_k}^2 \leq \|f^*\|_{\cH_k}^2,
$
So we have $\by^TK^{-1}\by\leq \|f^*\|^2_{\cH_k}$.
\end{proof}
The following lemma provides an upper bound for the minimum-norm solution of the random feature model \eqref{eqn: rf-min-norm}. 
\begin{lemma}
Assume that $f^*\in \cH_k$ with $k(\bx,\bx')=\EE_{\bw}[\varphi(\bx;\bw)\varphi(\bx';\bw)]$.
Then the minimum-norm estimator satisfies
$$
{\frac{1}{\sqrt{m}}}\|\hat{\ba}_n\|\leq 2 \|f^*\|_{\cH_k}
$$
\end{lemma}

\begin{proof}
    Let $\Phi=(\Phi_{i,j})\in \RR^{n\times m}$ with $\Phi_{i,j}=\varphi(\bx_i;\bw_j)$. Then the solution of problem \eqref{eqn: rf-min-norm} is given by 
    \[
        \hat{\ba}_n = {m}\Phi^T(\Phi\Phi^T)^{-1}\by.
    \]
    Obviously, $K^m = \frac{1}{m}\Phi\Phi^T$.
    Therefore, we have 
    \begin{align*}
        {\frac{1}{m}} \|\hat{\ba}_n\|^2 &=  m \by^T(\Phi\Phi^T)^{-1} \Phi \Phi^T(\Phi\Phi^T)^{-1}\by = \by^T (\frac{1}{m}\Phi\Phi^T)^{-1}\by = \by^T (K^m)^{-1}\by \\
        &= \by^T K^{-1} \by + \by^T ((K^m)^{-1} - K^{-1}) \by \\
        &=\by^T \hat{K}^{-1} \by + \by^T (K)^{-1}(K^m - K) (K^m)^{-1} \by \\
        &\leq \by^T \hat{K}^{-1} \by  + \|(K^m)^{-1/2}\by\| \|(K^m)^{-1/2}(K-K^m)K^{-1/2}\| \|K^{-1/2}\by\|
    \end{align*}
    According to Lemma \ref{lemma: bound of krr}, we have $\by^TK^{-1}\by\leq \|f^*\|_{\cH_k}^2$. Denote $t=\sqrt{\|\hat{\ba}_n\|^2/{ m}}=\sqrt{\by^T(K^m)^{-1}\by}$, we have 
    \[
        t^2 \leq \|f^*\|_{\cH_k}^2 + t\|f^*\|_{\cH_k}  \|(K^m)^{-1/2}\|\|K-K^m\|\|K^{-1/2}\|.
    \]
    By Lemma \ref{eqn: kernel-approx},  we have 
    \begin{align}
     t^2 \leq \|f^*\|_{\cH_k}^2 + t\|f^*\|_{\cH_k}  \sqrt{\frac{n^2\ln(2n^2/\delta)}{\lambda_n^2(K) m}}.
    \end{align}
    Under the assumption that $m\geq \frac{2n^2\ln(2n^2)}{\lambda_n^2(K)}$, we obtain
    \begin{equation*}
        { \frac{1}{\sqrt{m}}}\|\hat{\ba}_n\|  = t\leq 2 \|f^*\|_{\cH_k}.
    \end{equation*}
\end{proof}

\paragraph*{Proof of Theorem \ref{thm: random-feature}}
    Define  $A_C = \{\ba : { \frac{1}{\sqrt{m}}}\|\ba\|\leq C\}$ and $\cF_C = \{ f_m(\cdot;\ba) | \ba\in A_C \}$. The Rademacher complexity of $\cF_C$ satisfies, 
    \begin{align*}
    \rad_n(\cF_C) &= \frac{1}{n}\EE_{\xi_i}\sup_{f\in\cF_C}\sum_{i=1}^n \xi_i {\frac{1}{m}}\sum_{j=1}^m a_{j}\varphi(\bx_i;\bw_j) = \frac{1}{nm}\EE_{\xi_i}\sup_{f\in\cF_C}\sum_{j=1}^m a_{j} \sum_{i=1}^n \xi_i \varphi(\bx_i;\bw_j)\\
    &\leq \frac{1}{n{m}}\EE_{\xi_i}\sup_{f\in\cF_C}\sqrt{\sum_{j=1}^m a^2_{j}} \sqrt{\sum_{j=1}^m\left(\sum_{i=1}^n \xi_i \varphi(\bx_i;\bw_j)\right)^2}\\
    &\leq \frac{C}{n\sqrt{m}}\EE \sqrt{\sum_{j=1}^m\left(\sum_{i=1}^n \xi_i \varphi(\bx_i;\bw_j)\right)^2}\\
    &\stackrel{(i)}{\leq} \frac{C}{n\sqrt{m}}\sqrt{\EE [\sum_{j=1}^m\left(\sum_{i=1}^n \xi_i \varphi(\bx_i;\bw_j)\right)^2]} \\
    & = \frac{C}{n\sqrt{m}}\sqrt{\sum_{j=1}^m\sum_{i,i'=1}^n \EE[\xi_i\xi_{i'}] \varphi(\bx_i;\bw_j)\varphi(\bx_{i'};\bw_j)} \stackrel{(ii)}{\leq}\frac{C}{\sqrt{n}},
    \end{align*}
    where $(i)$ and $(ii)$ follow from the Jensen's inequality and  $\EE[\xi_i\xi_j]=\delta_{i,j}$, respectively.
    Moreover, for any $\ba \in A_C$, we have $|f_m(\bx;\ba)| = |{ \frac{1}{m}}\sum_{i=1}^m a_j \varphi(\bx;\bw_j)|\leq { \frac{1}{m}} \sqrt{ \|\ba\|_2^2 \sum_{j=1}^m \varphi(\bx;\bw_j)^2}\leq C$. Thus, for any $f_m(\cdot;\ba)\in \cF_C$,   
    the loss function $(f_m(\bx;\ba)-f(\bx))^2/2$ is $(C+1)-$Lipschitz continuous and bounded above  by $(C+1)^2/2$.  

    Take $C=2\|f^*\|_{\cH_k}$. We have $f_m(\cdot;\hat{\ba}_n)\in \cF_C$. Thus, Theorem \ref{thm: rad-gen-err} implies 
    \begin{align}
        R(\hat{\ba}_n) &\leq \hat{R}_n(\hat{\ba}_n)  + 2 (C+1)\rad_n(\cF_C) + \frac{4(C+1)^2}{2} \sqrt{\frac{2\ln(2/\delta)}{n}}\\
        &\lesssim \frac{\|f^*\|_{\cH_k}^2+1}{\sqrt{n}}\left(1 + \sqrt{\ln(2/\delta)}\right).
    \end{align}
\qed

\section{The two-layer neural network model}
First we recall the definition of the Barron space  \cite{e2019barron,e2018priori}.
\begin{definition}[Barron space]
Let $\bw=(\bb, c)$ and $\tilde{\bx}=(\bx^T,1)^T$.
Consider  functions that admit the following integral representation
\[
    f(\bx) = \EE_{\bw \sim \pi}[a(\bw)\sigma(\bw^T\tilde{\bx})],
\]
where $\pi$ is a probability measure over $\SS^{d}$ and $a(\cdot)$ is a measurable function. Denote $\Theta_f = \{(a,\pi) | f(\bx)=\EE_{\bw\sim\pi}[a(\bw)\sigma(\bw^T\tilde{\bx})]\}$,  the Barron norm is defined as follows 
\[
    \|f\|_{\cB} := \inf_{(a,\pi)\in \Theta_f}  \EE_{\bw\sim\pi} |a(\bw)|.
\]
The Barron space is defined as the set of continuous functions with a finite Barron norm, i.e. 
\[
    \cB := \{ f \,|\, \|f\|_{\cB}< \infty\}.
\]
\end{definition}

\begin{remark}
Let $k_{\pi}(\bx,\bx') := 
\EE_{\bw \sim \pi}[\sigma(\bw\cdot \tilde{\bx})\sigma(\bw\cdot\tilde{\bm{x}}')]$. In \cite{e2019barron}, it is proved that $\cB = \cup_{\pi \in P(\SS^d)} \cH_{k_{\pi}}$, where $P(\SS^d)$ denote the set of Borel probability measures on $\SS^d$.
Therefore the Barron space is much larger than the RKHS. 
\end{remark}

Let
\begin{equation}\label{eqn: minimum-norm-two-layer}
\hat{\theta}_n := \argmin_{\hat{R}_n(\theta)=0} \|\theta\|_{\cP},
\end{equation}
where  $\|\cdot\|_{\cP}$ is the discrete analog of the Barron norm (also known as the path norm): 
\begin{equation}
\|\theta\|_{\cP} := {\frac{1}{m}}\sum_{j=1}^m |a_j|(\|\bb_j\|_1+|c_j|).
\end{equation}
The generalization properties of the above estimator is given by the following theorem.
\begin{theorem} \label{thm: two-layer-net}
If  $m\geq \frac{8n^2\ln(2n^2)}{\lambda_n^2}$, then for any $\delta \in (0,1)$, with probability at least $1-\delta$ over the random choice of the training data, we have 
\[
    R(\hat{\theta}_n) \lesssim \frac{\|f^*\|^2_{\cB}+1}{\sqrt{n}}\left(\sqrt{\ln(2d)}+\sqrt{\ln(2/\delta)}\right)
\]
\end{theorem}

Before proving the main result, we need the following lemma.
\begin{lemma}\label{lem: fit-rand-label}
For any $\br\in \RR^n$, there exists a two-layer neural network $f_m(\cdot;\theta)$ with $m\geq \frac{2n^2\ln(4n^2)}{\lambda_n^2}$, such that   $f_m(\bx_i;\theta) = r_i $ for any $i\in [n]$ and   $\|\theta\|_{\cP} \leq \sqrt{\frac{2}{\lambda_n}}\|\br\|$
\end{lemma}
\begin{proof}
Assume that $\{(\bb_j,c_j)\}_{j=1}^m$ are i.i.d. random variables drawn from $\pi_0$, the uniform distribution over the sphere $\SS^d$.
 Recall that  $K^m:=(K^m_{i,i'})\in\RR^{n\times n}$ with
\[
    K^m_{i,i'} = \frac{1}{m}\sum_{j=1}^m \sigma(\bb_j\cdot\bx_i+c_j)\sigma(\bb_j\cdot\bx_{i'}+c_j)
\]
For any $\delta\in (0,1)$, if $m\geq \frac{2n^2\ln(2n^2/\delta)}{\lambda_n^2}$, Lemma \ref{eqn: kernel-approx} implies that the following hold with probability at least $1-\delta$
\begin{equation}\label{eqn: low-bound-eigen}
    \lambda_{n}(K^m)\geq \frac{1}{2}\lambda_n.
\end{equation}

Taking $\delta=1/2$, the above inequality holds with probability $1/2$. This means that there must exist $\{(\hat{\bb}_j,\hat{c}_j)\}_{j=1}^m$ such that \eqref{eqn: low-bound-eigen}  holds.
Let $\Psi\in\RR^{n\times m}$, $\Psi_{i,j} = \sigma(\hat{\bb}_{j}\cdot \bx_i+\hat{c}_j)$,  denote the feature matrix. Then
\begin{equation}\label{eqn: 22}
    \sigma_n^2(\Psi)=\lambda_n(\Psi\Psi^T) = m \lambda_n(K^m)\geq \frac{1}{2}\lambda_n m.
\end{equation}
We next choose $\ba$ as the solution of the following problem.
\begin{align*}
   \hat{\ba} =& \argmin\, \, \|\ba\| \\
    &\text{s.t.}\, { \frac{1}{\sqrt{m}}}\Psi\ba = \br.
\end{align*}
Then
\begin{equation}\label{eqn: 11}
\|\hat{\ba}\|\leq \sigma^{-1}_n(\Psi)\|\br\|\leq \sqrt{\frac{2}{\lambda_n}} \|\br\|.
\end{equation}
Consider the two-layer neural network 
\[
    f_m(\bx;\hat{\theta}) = {\frac{1}{m}\sum_{j=1}^m \hat{a}_j \sigma(\hat{\bb}_j\cdot \bx+ \hat{c}_j)}.
\]
Then we have that $f_m(\bx_j;\hat{\theta}) = r_j$ and 
$
\|\theta\|_{\cP} \leq { \frac{1}{m}}\sum_{j=1}^m |\hat{a}_j|\leq \|\hat{\ba}\|\leq \sqrt{\frac{2}{\lambda_n}} \|\br\|,
$
where the last inequality follows from \eqref{eqn: 11}.
\end{proof}

The following lemma provides an upper bound to the minimum path norm solutions.
\begin{lemma}\label{lemma: two-layer-m}
Assume that $f^*\in \cB$ and $m\geq \frac{6n^2\ln(4n^2)}{\lambda_n^2}$, then  the minimum path norm solution \eqref{eqn: minimum-norm-two-layer} satisfies 
\[
    \|\hat{\theta}_n\|_{\cP} \leq 3\|f^*\|_{\cB}.
\]
\end{lemma}

\begin{proof}
First by the approximation result of two-layer neural networks (See Proposition 2.1 in \cite{e2018priori}), for any $m>0$, there must exists a two-layer neural network $f_{m_1}(\cdot;\theta^{(1)})$ such that
\begin{equation}
\hat{R}_n(\theta^{(1)})=\|f_{m_1}(\cdot;\theta^{(1)}) - f^*\|^2_{\hat{\rho}} \leq \frac{3\|f^*\|_{\cB}^2}{m_1},
\end{equation}
and
\[
\|\theta^{(1)}\|_{\cP}\leq 2 \|f^*\|_{\cB}.
\]
where $\hat{\rho}(\bx)=\frac{1}{n}\sum_{i=1}^n \delta(\bx-\bx_i)$.

Let $\br=(y_1 - f_{m_1}(\bx_1;\theta^{(1)}),\dots,y_n - f_{m_1}(\bx_n;\theta^{(1)}))\in\RR^n$ to be the residual. Then $\|\br\|\leq \sqrt{\frac{3n}{m_1}}\|f^*\|_{\cB}$. Applying Lemma \ref{lem: fit-rand-label}, we know that there exists a two-layer neural network $f_{m_2}(\cdot;\theta^{(2)})$  with $m_2\geq \frac{2n^2\ln(4n^2)}{\lambda_n^2}$ such that
\begin{align}
    f_{m_2}(\bx_i;\theta^{(2)}) &= r_i \\
    \|\theta^{(2)}\|_{\cP} &\leq \sqrt{\frac{2}{\lambda_n}}\|\br\|\leq \|f^*\|_{\cB},
\end{align}
where the last inequality holds as long as $m_1\geq \frac{6n}{\lambda_n}$.

Putting $f_{m_1}(\cdot;\theta^{(1)}), f_{m_2}(\cdot;\theta^{(2)})$ together, let 
\[
    f_{m_1+m_2}(\bx;\theta) = f_{m_1}(\bx;\theta^{(1)})+f_{m_2}(\bx;\theta^{(2)}),
\]
where $\theta = \{\theta^{(1)},\theta^{(2)}\}$. It is obviously that 
\[
    \hat{R}_n(\theta)=0, \qquad \|\theta\|_{\cP} = \|\theta_1\|_{\cP}+\|\theta_2\|_{\cP} \leq 3 \|f\|_{\cB}.
\]
\end{proof}

\paragraph*{Proof of Theorem \ref{thm: two-layer-net}}
For any $C>0$, let $\cF_C = \{f_m(\bx;\theta) : \|\theta\|_{\cP}\leq C\}$. Using Lemma 4 of \cite{e2018priori}, we have $\rad_n(\cF_c)\leq 2C\sqrt{\frac{2\ln(2d)}{n}}$.  By  the definition of the minimum-norm solution and Lemma \ref{lemma: two-layer-m}, we have 
\[
    \|\hat{\theta}_n\|_{\cP} \leq 3\|f^*\|_{\cB}.
\]
Taking $C=3\|f^*\|_{\cB}$,  then we have $f_m(\cdot;\hat{\theta}_n)\in \cF_C$.  Since the loss function is $(C+1)$-Lipschitz continuous and bounded from above by $(C+1)^2/2$,
 by Theorem \ref{thm: rad-gen-err}, for $\delta \in (0,1)$, the following holds with probability at least $1-\delta$
\begin{align}
    R(\hat{\theta}_n)& \leq R_n(\hat{\theta}_n) + 2(C+1) \rad_n(\cF_C) + \frac{4}{2}(C+1)^2\sqrt{\frac{2\ln(2/\delta)}{n}} \\
    &\lesssim \frac{\|f^*\|^2_{\cB}+1}{\sqrt{n}}\left(\sqrt{\ln(2d)}+\sqrt{\ln(2/\delta)}\right)
\end{align}
\qed

\section{The residual neural network models}

First we recall the definition of the flow-induced function spaces $\cD_p$ \cite{e2019barron}. 

Let $\{\rho_t\}_{t\in [0,1]}$ be a family of Borel probability measures over $\RR^{D\times m}\times \RR^{m\times D}$. Consider  functions $f_{\balpha, \{\rho_t\}}$ defined through the following ordinary differential equations (ODE),
\begin{equation}\label{eqn: ode-fun}
\begin{aligned}
\bz(\bx,0) &= V\tilde{\bx} \\
\frac{d \bz(\bx, t)}{dt} &=  \EE_{(U,W)\sim \rho_t}[U\sigma(W\bz(\bx,t))] \\
f_{\balpha, \{\rho_t\}}(\bx) &= \balpha^T \bz(\bx, 1),
\end{aligned}
\end{equation}
where $V\in\RR^{D\times (d+1)}, U\in \RR^{D\times m}, V\in \RR^{m\times D}$ and $\balpha\in\RR^{D}$. The ODE \eqref{eqn: ode-fun} can be viewed as the continuous limit of the residual network \eqref{eqn: resnet}.
To define the norm for controlling the complexity of the flow map of ODE \eqref{eqn: ode-fun}, we need the following linear ODE 
\begin{align*}
\bn_p(0) &= |V| \bm{1}_{d+1} \\
\frac{d \bn_p(t)}{dt} &= 3 \, (\EE_{(U,W)\sim \rho_t}[(|U| |W|)^p])^{1/p} \bn_p(t),
\end{align*}
where $A^{q} = (a_{i,j}^q)$ for $A=(a_{i,j})$. {Specifically, $p=1,2$ are used in this paper.}

\begin{definition}[Flow-induced function space]
For a function $f$ that can be represented in the form \eqref{eqn: ode-fun}, we define 
\begin{equation}\label{eqn: composition-norm}
  \|f\|_{\cD_p} = \inf_{f=f_{\balpha,\{\rho_t\}}} |\balpha|^T \bn_p(1) + \|\bn_p(1)\|_1-D,
\end{equation}
to be the {``$\cD_p$ norm''} of $f$. The space $\cD_p$ is defined as the set of all  functions that admit the representation $f_{V,\{\rho_t\}}$ with finite $\cD_p$ norm.
\end{definition}
\begin{remark}
It should be noted that the function space $\cD_p$ actually depends on $D, m$. We use $\cD_p^{D,m}$ to explicitly show this dependence when it is needed. In most cases, 
this dependence is omitted in the notation of $\cD_p$ for simplicity. 
\end{remark}
In addition, we define the following norm to quantify the continuity of the sequence of probability measure $\{\rho_t\}_{t\in [0,1]}$.
\begin{definition}\label{def: lip}
Given a family of probability distribution $\{\rho_t\}_{t\in [0,1]}$, let $S(\{\rho_t\})$ denote the set of positive values $C>0$ that satisfies
\begin{equation}\label{eq:lip}
\left| \EE_{\rho_t} U\sigma( W\bz)-\EE_{\rho_s} U\sigma(W\bz) \right|\preceq C |t-s||\bz|, 
\end{equation}
and 
\begin{equation}\label{eq:lip2}
\left| \left\|\EE_{\rho_t}| U||W|\right\|_{1,1}-\left\|\EE_{\rho_s}| U|| W|\right\|_{1,1} \right| \leq C |t-s|,
\end{equation}
for any $t,s\in[0,1]$ and $\bz\in\RR^D$. We define the ``Lipschitz norm'' of $\{\rho_t\}_{t\in[0,1]}$ by
\begin{equation}
\|\{\rho_t\}\|_{\lip}=\left\|\EE_{\rho_{0}}|U||W|\right\|_{1,1}+\inf_{C\in S(\{\rho_t\})} C.
\end{equation}
\end{definition}

To control the complexity of a residual network, we use the following weighted path norm defined in \cite{ma2019priori}, which can be vied as an discrete analog of \eqref{eqn: composition-norm}.
\begin{definition}\label{eqn: composition-path-norm}
For any residual network $f_L(\cdot;\theta)$ given by \eqref{eqn: resnet}, its weighted path norm is define as,
\end{definition}
\begin{equation}
\|\theta\|_{\cC}:= |\balpha|^T \prod_{l=1}^L (I+{ \frac{3}{L}}|U_l||W_l|) |V| \bm{1}_{d+1}.
\end{equation}

We can now define the   minimum-norm estimator for residual neural networks: 
\begin{align}\label{def: minimum-norm-residual-net}
\hat{\theta}_n := \argmin_{\hat{R}_n(\theta)=0} \|\theta\|_{\cC},
\end{align}

\begin{theorem}\label{thm: resnet}
Assume that the target function $f^*\in \cD_2^{D,m}$ and  $c_0(f^*):=\inf_{f_{V, \{\rho_t\}}=f^*}\|\{\rho_t\}\|_{\lip}<\infty$. 
If the model is a $(L, D+d+2, m+1)$ residual neural network with the depth satisfying
\[
    L\geq C\max\left((m^4D^6c_0^2(f^*) \|f^*\|_{\cD_1}^2)^{6},\left(\frac{96nm^2}{\lambda_n}\right)^{\frac{3}{2}}, \frac{n (1+D)}{\lambda_n}, \frac{n^2\ln(2n)}{\lambda_n^2}\right),
\]
where $C$ is a universal constant.
Then for any $\delta \in (0,1), $with probability $1-\delta$ over the choice of the training data, we have  
\[
    R(\hat{\theta}_n)\lesssim \frac{\|f^*\|^2_{\cD_1}+1}{\sqrt{n}} \left(\sqrt{\ln(2d)} + \sqrt{\ln(2/\delta)}\right)
\]
\end{theorem}

The following lemma shows that the addition of two residual networks can be represented by a wider residual network.
\begin{lemma}\label{lem: resnet-add}
Suppose that $f(\cdot;\theta^{(1)})$ and $f(\cdot;\theta^{(2)})$ are $(L_1,D_1,m_1)$ and $(L_2, D_2, m_2)$ residual networks, respectively. Then $F:=f(\cdot;\theta_1)+f(\cdot;\theta_2)$ can be represented as a $(\max(L_1,L_2), D_1+D_2, m_1+m_2)$ residual network $\theta^{(3)}$ and the weighted path norm satisfies
\[
    \|\theta^{(3)}\|_{\cC}  = \|\theta^{(1)}\|_{\cC} + \|\theta^{(2)}\|_{\cC}.
\]
\end{lemma}
\begin{proof}
Without loss of generality, we assume $L_1=L_2$. Otherwise, we can add extra identity layers without changing the represented function and the path norm.  $f(\cdot;\theta^{(1)})$ can be written as 
\begin{equation*}
\begin{aligned}
\bz^{(1)}_0(\bx) &= V^{(1)}\tilde{\bx} \\
\bz^{(1)}_{l+1}(\bx) &= \bz^{(1)}_{l}(\bx) + {\frac{1}{L}}U_l^{(1)} \sigma(W_l^{(1)}\bz^{(1)}_{l}(\bx)) , \quad l=0,\dots, L-1\\
f(\bx;\theta^{(1)}) &= (\balpha^{(1)})^T\bz^{(1)}_L(\bx),
\end{aligned}
\end{equation*}
where $U_l^{(1)}\in \RR^{D_1\times m_1}, W^{(1)}_l\in \RR^{m_1\times D_1}, V^{(1)}\in \RR^{D_1\times (d+1)}, \balpha^{(1)}\in\RR^{D_1}$. Similarly, for $f(\cdot;\theta^{(2)})$, we have 
\begin{equation*}
\begin{aligned}
\bz^{(2)}_0(\bx) &= V^{(2)}\tilde{\bx} \\
\bz^{(2)}_{l+1}(\bx) &= \bz^{(2)}_{l}(\bx) + {\frac{1}{L}}U_l^{(2)} \sigma(W_l^{(2)}\bz^{(2)}_{l}(\bx)) , \quad l=0,\dots, L-1\\
f(\bx;\theta^{(2)}) &= (\balpha^{(2)})^T\bz^{(2)}_L(\bx),
\end{aligned}
\end{equation*}
where $U_l^{(2)}\in \RR^{D_2\times m_2}, W^{(2)}_l\in \RR^{m_2\times D_2}, V^{(2)}\in \RR^{D_1\times (d+1)}, \balpha^{(2)}\in\RR^{D_2}$.

Let 
\begin{align}
V = \begin{bmatrix}
V^{(1)}\\
V^{(2)}
\end{bmatrix},
U_l = \begin{bmatrix}
U^{(1)}_l & 0 \\
0 & U^{(2)}_l 
\end{bmatrix},
V_l = \begin{bmatrix}
V^{(1)}_l & 0 \\
0 & V^{(2)}_l 
\end{bmatrix},
\balpha = \begin{bmatrix}
\balpha^{(1)} \\
\balpha^{(2)}
\end{bmatrix}.
\end{align}
Consider the following residual network 
\begin{equation*}
\begin{aligned}
\bz_0(\bx) &= V\tilde{\bx} \\
\bz_{l+1}(\bx) &= \bz_{l}(\bx) + {\frac{1}{L}}U_l \sigma(W_l\bz_{l}(\bx)) , \quad l=0,\dots, L-1\\
f(\bx;\theta^{(3)}) &= \balpha^T\bz_L(\bx),
\end{aligned}
\end{equation*}
where $\bz_l(\bx)\in \RR^{D_1+D_2}$.  Here $f(\cdot;\theta^{(3)})$ is a $(L, D_1+D_2, m_1+m_2)$ residual network and 
it is easy to show that 
\begin{align*}
    f(\bx;\theta^{(3)}) &= f(\bx; \theta^{(1)}) + f(\bx;\theta^{(2)}) \\
    \|\theta^{(3)}\|_{\cC} &= \|\theta^{(1)}\|_{\cC} + \|\theta^{(2)}\|_{\cC}.
\end{align*}

\end{proof}

The following lemma shows that any two-layer neural network can be converted to an residual network, without changing the norm too much.
\begin{lemma}\label{lemma: embedding}
For any two layer neural network $f_m(\cdot;\theta)$ of width $m$. There exists a $(m, d+2, 1)$ residual network $g_{m}(\cdot; \Theta)$ such that 
\begin{align*}
    g_{m}(\bx;\Theta) &= f_{m}(\bx;\theta) \quad \forall \bx \in \RR^d \\
    \|\Theta\|_{\cC} & = 3 \|\theta\|_{\cP}.
\end{align*}
\end{lemma}
\begin{proof}
Assume the two-layer neural network is given by $f_m(\bx;\theta)={\frac{1}{m}}\sum_{j=1}^m a_j \sigma(\bb_j\cdot\bx+c_j)$. Consider the following residual network,
\begin{align*}
\bz_0(\bx) &= 
\begin{pmatrix}
I_{d+1}\\
0 
\end{pmatrix} \tilde{\bx}\\
\bz_{j+1}(\bx) &= \bz_{j}(\bx) + {\frac{1}{m} }\begin{pmatrix}
\bm{0}_d \\
0 \\
a_j
\end{pmatrix} 
\sigma(
\begin{pmatrix}
\bb_j^T & c_j & 0
\end{pmatrix} \bz_j(\bx)
), \quad j=1,\dots, m-1\\
g_m(\bx;\Theta) &= \be_{d+2}^T \bz_{m}(\bx),
\end{align*}
where $\be_{d+2}=(0,\dots, 0,1)^T\in\RR^{d+2}$. Obviously, $g_m(\bx;\Theta)=f_m(\bx;\theta)$ for any $\bx\in\RR^d$. Moreover, the weighted path norm satisfies 
\begin{align}\label{eqn: res-norm}
  \nonumber  \|\Theta\|_{\cC} & = 
  \be_{d+2}^T \left[\Pi_{l=1}^M \left(I + {\frac{3}{m}}\begin{pmatrix}
\bm{0}_d \\
0 \\
|a_l|
\end{pmatrix} \begin{pmatrix}
|\bb_l|^T & |c_l| & 0
\end{pmatrix} \right)\right] 
\begin{pmatrix}
I_{d+1} \\
0
\end{pmatrix}
\bm{1}_{d+1} \\
      &= {\frac{3}{m}}\sum_{j=1}^M |a_j|(\|\bb_j\|_1+|c_j|) = \|\theta\|_{\cP}.
\end{align}
\end{proof}

\paragraph*{Proof of Theorem \ref{thm: resnet}}
By the direct approximation theorem (Theorem 10 in \cite{e2019barron}), for any  $\delta_0\in (0,1)$, there exists a $L_1=(m^4D^6c_0^2(f^*) \|f^*\|_{\cD_1}^2)^{3/\delta_0}$, such that for any $L\geq L_1$, there exists a $(L, D, m)$ residual network $f_{L}(\cdot; \theta^{(1)})$ such that 
\begin{equation}\label{eqn: app-rate-res}
\begin{aligned}
    \hat{R}_n(\theta^{(1)})& =\|f^* - f_L(\cdot;\theta^{(1)})\|^2_{\hat{\rho}_n}\\
    &\leq \frac{24m^2}{L^{1-2\delta_0/3}}\|f^*\|_{\cD_1}^4 + \frac{3C}{L}(1+D+\sqrt{\log L})\|f^*\|^2_{\cD_1},
\end{aligned}
\end{equation}
and 
\begin{align}\label{eqn: res1}
    \|\theta^{(1)}\|_{\cC} \leq 9 \|f^*\|_{\cD_1},
\end{align}
where $C$ is a universal constant. 

Let $\br=(y_1 - f_L(\bx_1;\theta^{(1)}), \dots, y_n - f_L(\bx_n;\theta^{(1)}))^T\in \RR^n$ to be the residual.    
 $
 \|\br\| = \sqrt{n\hat{R}_n(\theta^{(1)})}.
 $
By Lemma \ref{lem: fit-rand-label}, there exists a two-layer neural network $h_M(\bx; \theta)={ \frac{1}{M}} \sum_{j=1}^M a_j \sigma(\bb_j^T\bx+c_j)$ of $M = \frac{2n^2\ln(4n^2)}{\lambda_n^2}$ such that $h_M(\bx_i;\theta)=r_i$ and 
\begin{align*}
{\frac{1}{M}}\sum_{j=1}^M |a_j| (\|\bb_j\|_1+|c_j|)  &\lesssim \sqrt{\frac{2}{\lambda_n}}\|\br\|.
\end{align*}
Inserting \eqref{eqn: app-rate-res} gives us 
\begin{align}\label{eqn: bound-residual}
\nonumber {\frac{1}{M}} \sum_{k=1}^M |a_k| (\|\bb_k\|_1+|c_j|) &\leq \sqrt{\frac{2n}{\lambda_n}\left( \frac{24m^2}{L^{1-2\delta_0/3}}\|f^*\|_{\cD_1}^2 + \frac{3C}{L}(1+D+\sqrt{\log L})\right)}\|f^*\|_{\cD_1}\\
&\leq \|f^*\|_{\cD_1},
\end{align}
where the last inequality holds as long as $L\geq \max((96m^2n/\lambda_n)^{3/(3-2\delta_0)}, 12Cn(1+D+\sqrt{\log L})/\lambda_n)$. 
By Lemma \ref{lemma: embedding}, there exists a $(M,d+1,1)$ residual network $f_M(\cdot;\theta^{(2)})$ such that $f_M(\bx_i;\theta^{(2)})= h_M(\bx_i;\theta)=r_i$ and 
\[
\|\theta^{(2)}\|_{\cC} =  {\frac{3}{M}}\sum_{j=1}^M |a_j|(\|\bb_j\|_1+|c_j|)\leq 3 \|f^*\|_{\cD_1}.
\]

Note that $L\geq M$. Applying Lemma \ref{lem: resnet-add}, we conclude that $f_L(\cdot;\theta^{(1)}) + f_M(\cdot;\theta^{(2)})$ can be represented by a $(L, D+d+2, m+1)$ residual network $f_L(\cdot;\theta^{(3)})$, which satisfies 
\begin{align*}
    \hat{R}_n(\theta^{(3)}) & = 0 \\
      \|\theta^{(3)}\|_{\cC} &= \|\theta^{(1)}\|_{\cC} + \|\theta^{(2)}\|_{\cC} \leq 12\|f^*\|_{\cD_1},
\end{align*}
where the last inequality follows from  \eqref{eqn: res1} and \eqref{eqn: res-norm}. By the definition of the minimum-norm solutions \eqref{def: minimum-norm-residual-net}, we have 
\[
    \|\hat{\theta}_n\|_{\cC} \leq \|\theta_3\|_{\cC}\leq 12 \|f^*\|_{\cD_1}.
\]

Let $\cF_C=\{f_L(\cdot;\theta) : \|\theta\|_{\cC}\leq C\}$ denote the set of $(L,D+d+2,m+1)$ residual network with the weighted path norm bounded from above by $C$. Theorem 2.10 of \cite{ma2019priori} states that 
\[
    \rad_n(\cF_C)\leq 3C \sqrt{\frac{2\log(2d)}{n}}.
\]
For any $f\in \cF_C$, we have $|f|\leq C$, therefore the loss function  is $(C+1)$-Lipschitz continuous and bounded by $(C+1)^2/2$. 
Taking $C=12\|f^*\|_{\cD_1}$, then we have $f_L(\cdot;\hat{\theta}_n)\in \cF_C$. 
Applying Theorem \ref{thm: rad-gen-err},  we conclude that with probability at least $1-\delta$ over the sample of the training set, we have 
\begin{align*}
R(\hat{\theta}_n)& \leq \hat{R}(\hat{\theta}_n) + 2 (C+1)\rad_n(\cF_C) + \frac{4(C+1)^2}{2}\sqrt{\frac{2\ln(2/\delta)}{n}}\\
& \lesssim  \frac{\|f^*\|^2_{\cD_1}+1}{\sqrt{n}} \left(\sqrt{\ln(2d)} + \sqrt{\ln(2/\delta)}\right),
\end{align*}
where the last inequality holds since $C=12\|f^*\|_{\cD_1}$. Taking $\delta_0=1/2$ completes the proof.
\qed

\section{Concluding Remarks}
In this work, we prove that learning with the minimum-norm interpolation scheme can achieve Monte Carlo error rates  for three models: the random feature model, the two-layer neural network model
and the residual neural network model. The proofs rely on two assumptions: (1) the model is sufficiently over-parametrized; (2) the labels are clean, i.e. $y_i=f^*(\bx_i)$. 
The ``double descent'' phenomenon \cite{belkin2019reconciling} tells us the results are unlikely to be true when the
models are not over-parametrized.
When the data suffers from measurement noise, we also expect that the results will deteriorate.
However, recent work \cite{liang2018just,liang2019risk,zhang2016understanding} showed that for kernel regression, noise may not hurt the generalization error too much, especially in the high-dimensional regime.  It would be interesting to consider this issue for neural network models. We leave this to future work.

\subsection*{Acknowledgement}
The work presented here is supported in part by a gift to Princeton University from iFlytek
and the ONR grant N00014-13-1-0338.

\end{document}